\documentclass[11pt]{article}
\pdfoutput=1

\usepackage{amsmath,amssymb}
\usepackage{bm}
\usepackage{natbib}
\usepackage[usenames]{color}
\usepackage{amsthm}

\usepackage{multirow} 
\usepackage{enumerate}

\allowdisplaybreaks

\usepackage[colorlinks,
linkcolor=red,
anchorcolor=blue,
citecolor=blue
]{hyperref}

\usepackage{setspace}
\usepackage[left=1in, right=1in, top=1in, bottom=1in]{geometry}

\usepackage{xcolor}
\usepackage{footnote}

\usepackage[utf8]{inputenc} 
\usepackage[T1]{fontenc}    
\usepackage{hyperref}       
\usepackage{url}            
\usepackage{booktabs}       
\usepackage{multirow}
\usepackage{amsfonts}       
\usepackage{nicefrac}       
\usepackage{microtype}      

\usepackage{booktabs}       
\usepackage{multirow}
\usepackage{nicefrac}       
\usepackage{microtype}      

\usepackage{graphicx}

\usepackage{bbm} 

\theoremstyle{definition}
\newtheorem{theorem}{Theorem}[section]
\newtheorem{corollary}[theorem]{Corollary}
\newtheorem{lemma}[theorem]{Lemma}
\newtheorem{proposition}[theorem]{Proposition}

\newtheorem{definition}[theorem]{Definition}
\newtheorem{remark}[theorem]{Remark}

\theoremstyle{definition}
\newtheorem{example}[theorem]{Example}

\newcommand{\N}{\mathbb{N}}

\newcommand{\R}{\mathbb{R}}

\newcommand{\E}{\mathbb{E}}

\let\P\BP

\let\hat\widehat

\newcommand{\f}{\frac}

\newcommand{\eps}{\varepsilon}
\renewcommand{\r}{\right}
\renewcommand{\l}{\left}

\newcommand{\pnorm}[2]{\left\|#1\right\|_{#2}}
\newcommand{\norm}[1]{\left\|#1\right\|}

\newcommand{\sip}[2]{\langle #1, #2 \rangle}
\newcommand{\ip}[2]{\left\langle #1, #2 \right\rangle}

\newcommand{\ddx}[1]{\frac{\mathrm{d}}{\mathrm{d} #1}}

\newcommand{\ind}{{\mathbbm{1}}}

\newcommand{\summ}[2]{\sum_{#1 = 1}^{#2}}
\newcommand{\summm}[3]{\sum_{#1 = #2}^{#3}}

\newcommand{\sgn}{\operatorname{sgn}}

\newcommand{\calD}{\mathcal{D}}

\newcommand{\calG}{\mathcal{G}}

\DeclareMathOperator{\poly}{poly}

\newcommand{\iid}{\stackrel{\rm i.i.d.}{\sim}}

\makeatletter
\renewcommand*{\eqref}[1]{%
  \hyperref[{#1}]{\textup{\tagform@{\ref*{#1}}}}%
}
\makeatother

\newcommand{\opt}{\mathsf{OPT}}

\numberwithin{equation}{section}

\newcommand{\obj}{F_{\ell}}

\newcommand{\sopt}[1]{\mathsf{OPT}^{-\ell'}}
\newcommand{\optz}{\mathsf{OPT}}

\newcommand{\llp}{H}

\newcommand{\zeroone}{\mathrm{err}_{\calD}^{0-1}}

\allowdisplaybreaks

\title{\huge Agnostic Learning of Halfspaces with Gradient Descent via Soft Margins}

\author
{
    Spencer Frei\thanks{Department of Statistics, University of California, Los Angeles, CA 90095, USA; e-mail: {\tt spencerfrei@ucla.edu}}
    ~~~and~~~
	Yuan Cao\thanks{Department of Computer Science, University of California, Los Angeles, CA 90095, USA; e-mail: {\tt yuancao@cs.ucla.edu}} 
	~~~and~~~
	Quanquan Gu\thanks{Department of Computer Science, University of California, Los Angeles, CA 90095, USA; e-mail: {\tt qgu@cs.ucla.edu}}
}
\date{}

\begin{document}
\maketitle

\begin{abstract}
    We analyze the properties of gradient descent on convex surrogates for the zero-one loss for the agnostic learning of halfspaces.  If $\mathsf{OPT}$ is the best classification error achieved by a halfspace, by appealing to the notion of soft margins we show that gradient descent finds halfspaces with classification error $\tilde O(\mathsf{OPT}^{1/2}) + \varepsilon$ in $\mathrm{poly}(d,1/\varepsilon)$ time and sample complexity for a broad class of distributions that includes log-concave isotropic distributions as a subclass.  To the best of our knowledge, this is the first positive guarantee for the classification error of halfspaces learned by gradient descent using either the binary cross-entropy or hinge loss in the presence of agnostic noise. 
    
    %
\end{abstract}


\section{Introduction}
We analyze the performance of gradient descent on a convex surrogate for the zero-one loss in the context of the agnostic learning of halfspaces.  By a \emph{halfspace} we mean a function $x\mapsto \sgn(w^\top x)\in \{\pm 1\}$ for some $w\in \R^d$.  Let $\calD$ be a joint distribution over $(x,y)$, where the inputs $x\in \R^d$ and the labels $y\in \{\pm 1\}$, and denote by $\calD_x$ the marginal of $\calD$ over $x$.  We are interested in the performance of halfspaces found by gradient descent in comparison to the best-performing halfspace over $\calD$, so let us define, for $w\in \R^d$,
\begin{align*}
\zeroone(w) &:= \P_{(x, y)\sim \calD}( \sgn(w^\top x) \neq y),\\
\optz &:= \min_{\norm w = 1} \zeroone(w).
\end{align*}
Due to the non-convexity and discontinuity of the zero-one loss, the standard approach for minimizing the classification error is to consider a convex surrogate loss $\ell : \R \to \R$ for which $\ind(z< 0) \leq O(\ell(z))$ and to instead minimize the surrogate risk
\begin{equation}
    \obj(w) := \E_{(x,y) \sim \calD} \big[\ell(y w^\top x)\big].
    \label{eq:objective.fcn}
\end{equation}
Without access to the population risk itself, one can take samples $\{(x_i, y_i)\}_{i=1}^n \iid \calD$ and optimize~\eqref{eq:objective.fcn} by gradient descent on the empirical risk $\hat \obj(w)$, defined by taking the expectation in \eqref{eq:objective.fcn} over the empirical distribution of the samples.  By using standard tools from convex optimization and Rademacher complexity, such an approach is guaranteed to efficiently minimize the population surrogate risk up to optimization and statistical error.  The question is then, given that we have found a halfspace $x\mapsto w^\top x$ that minimizes the \emph{surrogate} risk, how does this halfspace compare to the \emph{best} halfspace as measured by the zero-one loss?   And how does the choice of the surrogate loss affect this behavior?  To the best of our knowledge, no previous work has been able to demonstrate that gradient descent on convex surrogates can yield approximate minimizers for the classification error over halfspaces, even for the case of the standard logistic (binary cross-entropy) loss $\ell(z) = \log(1+\exp(-z))$ or the hinge loss $\ell(z) = \max(1-z, 0)$. 

We show below that the answer to these questions depend upon what we refer to as the \emph{soft margin function} of the distribution at a given minimizer for the zero-one loss.  (We note that in general, there may be multiple minimizers for the zero-one loss, and so we can only refer to \emph{a} given minimizer.)  For $\bar v\in \R^d$ satisfying $\norm {\bar v}=1$, we say that the halfspace $\bar v$ satisfies the $\phi_{\bar v}$-soft-margin property if for some function $\phi_{\bar v}: [0,1]\to \R$, for all $\gamma \in [0,1]$,
\begin{equation}\nonumber
    \P_{\calD_x}( |\bar v^\top x| \leq \gamma)\leq \phi_v(\gamma).
    \label{eq:softmargin.function}
\end{equation}
To get a flavor for how this soft margin can be used to show that gradient descent finds approximately optimal halfspaces, for bounded distributions $\calD_x$, we show in Theorem \ref{thm:bounded} below that with high probability,
\begin{align*} 
\zeroone(w_T) &\leq  \inf_{\gamma \in (0,1)} \Big\{ O( \gamma^{-1}  \optz) +  \phi_{\bar v}(\gamma)  + O( \gamma^{-1} n^{-1/2}) + \eps\Big\},
\end{align*}
where $\phi_{\bar v}$ is a soft margin function corresponding to a unit norm minimizer $\bar v$ of the population zero-one loss.  Thus, by analyzing the properties of $\phi_{\bar v}$, one can immediately derive approximate agnostic learning results for the output of gradient descent.   In particular, we are able to show the following guarantees for the output of gradient descent:
\begin{enumerate}
    \item \textbf{Hard margin distributions}.  If $\norm x \leq B_X$ almost surely and there is $\bar \gamma>0$ such that $\bar v^\top x \geq \bar \gamma$ a.s., then $\zeroone(w_{t}) \leq \tilde O(\bar \gamma^{-1} \optz) + \eps$. 
    \item \textbf{Sub-exponential distributions satisfying anti-concentration.}  If random vectors from $\calD_x$ are sub-exponential and satisfy an anti-concentration inequality for projections onto one dimensional subspaces, then $\zeroone(w_t)\leq \tilde O(\optz^{1/2}) + \eps$.  This covers any log-concave isotropic distribution.
\end{enumerate}
For each of our guarantees, the runtime and sample complexity are $\poly(d, \eps^{-1})$.  The exact rates are given in Corollaries \ref{corollary:agnostic.margin.bounded}, \ref{corollary:bounded} and \ref{corollary:log.concave.isotropic}.   In Table \ref{table:results} we compare our results with known lower bounds in the literature.   To the best of our knowledge, our results are the first to show that gradient descent on convex surrogates for the zero-one loss can learn halfspaces in the presence of agnostic label noise, despite the ubiquity of this approach for classification problems.

\begin{table*}[t]
\centering
\caption{Comparison of our results with other upper and lower bounds in the literature.}
\begin{tabular}[t]{p{4.4cm} p{3.3cm} p{2.8cm} p{4.4cm} }
\addlinespace
\toprule
Algorithm  & $\calD_x$ & Population Risk & Known Lower Bound\\
\midrule
Non-convex G.D.\newline \citep{diakonikolas2020nonconvex} & Concentration,\newline anti-concentration & $O(\optz)$ & N/A\\
\addlinespace
Convex G.D.\newline (this paper) & Sub-exponential,\newline anti-concentration & $\tilde O(\optz^{1/2})$ & $\Omega(\optz \log^\alpha (\nicefrac 1 \optz))$\newline \citep{diakonikolas2020nonconvex}\\
\addlinespace
Convex G.D.\newline (this paper) & Hard margin & $\tilde O(\bar \gamma^{-1} \optz) $ & $\Omega(\bar \gamma^{-1} \optz)$\newline \citep{diakonikolas2019massart}\\
\bottomrule
\addlinespace
\end{tabular}
\label{table:results}
\end{table*}


The remainder of the paper is organized as follows.  In Section \ref{sec:related.work}, we review the literature on learning halfspaces in the presence of noise.  In Section \ref{sec:soft.margins}, we discuss the notion of soft margins which will be essential to our proofs, and provide examples of soft margin behavior for different distributions.  In Section \ref{sec:gd.surrogate} we show that gradient descent efficiently finds minimizers of convex surrogate risks and discuss how the tail behavior of the loss function can affect the time and sample complexities of gradient descent.  In Section \ref{sec:gd.zeroone} we provide our main results, which relies upon using soft margins to convert minimizers for the convex surrogate risk to approximate minimizers for the classification error.  We conclude in Section \ref{sec:conclusion}. 

\section{Related Work}\label{sec:related.work}
The problem of learning halfspaces is a classical problem in machine learning with a history almost as long as the history of machine learning itself, starting from the perceptron~\citep{rosenblatt1958perceptron} and support vector machines~\citep{boser1992svm} to today.  Much of the early works on this problem focused on the realizable setting, i.e. where $\optz =0$.  In this setting, the Perceptron algorithm or methods from linear programming can be used to efficiently find the optimal halfspace.   In the setting of agnostic PAC learning~\citep{kearns.agnostic} where $\optz >0$ in general, the question of which distributions can be learned up to classification error $\optz + \eps$, and whether it is possible to do so in $\poly(d,1/\eps)$ time (where $d$ is the input dimension), is significantly more difficult and is still an active area of research.  It is known that without distributional assumptions, learning up to even $O(\optz) +\eps$ is NP-hard, both for proper learning~\citep{guruswami2009hardness} and improper learning~\citep{daniely2016complexity}.  Due to this difficulty, it is common to make a number of assumptions on either $\calD_x$ or to impose some type of structure to the learning problem.

A common structure imposed is that of structured noise: one can assume that there exists some underlying halfspace $y = \sgn(v^\top x)$ that is corrupted with probability $\eta(x)\in [0,\nicefrac 12)$, possibly dependent on the features $x$.  The simplest setting is that of random classification noise, where $\eta(x) \equiv \eta$, so that each label is flipped with the same probability~\citep{angluin1988rcn}; polynomial time algorithms for learning under this noise condition were shown by~\citet{blum1998rcn}.  The Massart noise model introduced by \citet{massart2006noise} relaxes this assumption to $\eta(x)\leq \eta$ for some absolute constant $\eta<1/2$.  The Tsybakov noise model~\citep{tsybakov2004noise} is a generalization of the Massart noise model that instead requires a tail bound on $\P(\eta(x)\geq  \nicefrac 12 - t)$ for $t>0$.~\citet{awasthi2015massart} showed that optimally learning halfspaces under Massart noise is possible for the uniform distribution on the unit sphere, and~\citet{awasthi20161bitcompressednoise} showed this for log-concave isotropic distributions. The recent landmark result of~\citet{diakonikolas2019massart} provided the first distribution-independent result for optimally learning halfspaces under Massart noise, answering a long-standing~\citep{sloan1988} open problem in computational learning. 

By contrast, in the agnostic PAC learning setting, one makes no assumptions on $\eta(x)$, so one can equivalently view agnostic PAC learning as an adversarial noise model in which an adversary can corrupt the label of a sample $x$ with any probability $\eta(x)\in [0,1]$.  Recent work suggests that even when $\calD_x$ is the Gaussian, agnostically learning up to exactly $\optz+\eps$ likely requires $\exp(1/\eps)$ time~\citep{goel2020sqlb,diakonikolas2020sqlb}. In terms of positive results in the agnostic setting,~\citet{kalai08agnostichalfspace} showed that a variant of the \textsf{Average} algorithm~\citep{servidio99average} can achieve risk $O(\optz \sqrt{\log(\nicefrac 1 \optz)})$ risk in $\poly(d,1/\eps)$ time when $\calD_x$ is uniform over the unit sphere.~\citet{awasthi} demonstrated that a localization-based algorithm can achieve $O(\optz)+\eps$ under log-concave isotropic marginals.~\citet{diakonikolas2020nonconvex} showed that for a broad class of distributions, the output of projected SGD on a nonconvex surrogate for the zero-one loss produces a halfspace with risk $O(\optz) + \eps$ in $\poly(d, 1/\eps)$ time.  For more background on learning halfspaces in the presence of noise, we refer the reader to~\citet{balcan2020noise}. 

We note that \citet{diakonikolas2020nonconvex} also showed that the minimizer of the surrogate risk of any \emph{convex} surrogate for the zero-one loss is a halfspace with classification error $\omega(\optz)$. ~\citet{bendavid2012surrogate} and~\citet{awasthi} showed similar lower bounds that together imply that empirical risk minimization procedures for convex surrogates yield halfspaces with classification error $\Omega(\optz)$.  Given such lower bounds, we wish to emphasize that in this paper we are \emph{not} making a claim about the optimality of gradient descent (on convex surrogates) for learning halfspaces.  Rather, our main interest is the characterization of what are the strongest learning guarantees possible with what is perhaps the simplest learning algorithm possible.  Given the success of gradient descent for the learning of deep neural networks, and the numerous questions that this success has brought to the theory of statistics and machine learning, we think it is important to develop a thorough understanding of what are the possibilities of vanilla gradient descent, especially in the simplest setting possible.  

Recent work has shown that gradient descent finds approximate minimizers for the population risk of single neurons $x\mapsto \sigma(w^\top x)$ under the squared loss~\citep{diakonikolas2020relu,frei2020singleneuron}, despite the computational intractability of finding the optimal single neuron~\citep{goel2019relugaussian}.   The main contribution of this paper is that despite the computational difficulties in \emph{exact} agnostic learning, the standard gradient descent algorithm satisfies an \emph{approximate} agnostic PAC learning guarantee, in line with the results found by~\citet{frei2020singleneuron} for the single neuron.

\subsection{Notation}
We say that a differentiable loss function $\ell$ is $L$-Lipschitz if $|\ell'(z)| \leq L$ for all $z$ in its domain, and we say the loss is $H$-smooth if its derivative $\ell'$ is $H$-Lipschitz.  We use the word ``decreasing'' interchangeably with ``non-increasing''.  We use the standard $O(\cdot), \Omega(\cdot)$ order notations to hide universal constants and $\tilde O(\cdot), \tilde \Omega(\cdot)$ to additionally suppress logarithmic factors.  Throughout this paper, $\norm{x}$ refers to the standard Euclidean norm on $\R^d$ induced by the inner product $x^\top x$.  We will emphasize that a vector $v$ is of unit norm by writing $\bar v$.  We assume $\calD$ is a probability distribution over $\R^d \times \{\pm 1\}$ with marginal distribution $\calD_x$ over $\R^d$.
\section{Soft Margins}\label{sec:soft.margins}
In this section we will formally introduce the soft margin function and describe some common distributions for which it takes a simple form.  

\begin{definition}\label{def:softmargin}
Let $\bar v\in \R^d$ satisfy $\norm{\bar v} =1$.  We say $\bar v$ satisfies the \emph{soft margin condition with respect to a function $\phi_{\bar v}:\R \to \R$} if for all $\gamma \in [0,1]$, it holds that
\[ \E_{x\sim \calD_x}\l[\ind\l(   |\bar v^\top x| \leq \gamma \r)\r]\leq \phi_{\bar v}(\gamma).\]
\end{definition}
We note that our definition of soft margin is essentially an unnormalized version of the soft margin function considered by~\citet{foster2018} in the context of learning GLMs, since they defined $\phi_{\bar v}(\gamma)$ as the probability that $|\bar v^\top x / \norm x| \leq \gamma$.  This concept was also considered by~\citet{balcan2017logconcave} for $s$-concave isotropic distributions under the name `probability of a band'.  

Below we will consider some examples of soft margin function behavior.  We shall see later that our final generalization bounds will depend on the behavior of $\phi_{\bar v}(\gamma)$ for $\gamma$ sufficiently small, and thus in the below examples we only care about the behavior of $\phi_{\bar v}(\cdot)$ in small neighborhoods of the origin.  In our first example, we show that (hard) margin distributions have simple soft margin functions.  
\begin{example}[Hard margin distributions]\label{example:hard.margin}
If $\calD_x$ is a hard margin distribution in the sense that $\bar v^\top x \geq \gamma^*>0$ for some $\gamma^*>0$ almost surely, then $\phi_{\bar v}(\gamma)=0$ for $\gamma <\gamma^*$.
\end{example}
\begin{proof}
This follows immediately: $\P(|\bar v^\top x| \leq \gamma)=0$ when $\gamma < \gamma^*$. 
\end{proof} 
Note that the soft margin function in Example \ref{example:hard.margin} is specific to the vector $\bar v$, and does not necessarily hold for arbitrary unit vectors in $\R^d$.  By contrast, for many distributions it is possible to derive bounds on soft margin functions that hold for \emph{any} vector $\bar v$, which we shall see below is a key step for deriving approximate agnostic learning guarantees for the output of gradient descent.

The next example shows that provided the projections of $\calD_x$ onto one dimensional subspaces satisfy an anti-concentration property, then all soft margins function for that distribution take a simple form.  To do so we first introduce the following definition.
\begin{definition}[Anti-concentration]\label{assumption:anti.concentration}
For $\bar v\in \R^d$, denote by $p_{\bar v}(\cdot)$ the marginal distribution of $x\sim \calD_x$ on the subspace spanned by $\bar v$.  We say $\calD_x$ satisfies \emph{$U$-anti-concentration} if there is some $U>0$ such that for all unit norm $\bar v$, $p_{\bar v}(z)\leq U$ for all $z\in \R$.
\end{definition}
A similar assumption was used in~\citet{diakonikolas2020massart,diakonikolas2020nonconvex,diakonikolas2020tsybakov} for learning halfspaces; in their setup, the anti-concentration assumption was for the projections of $\calD_x$ onto two dimensional subspaces rather than the one dimensional version we consider here. 

\begin{example}[Distributions satisfying anti-concentration]\label{example:anti.concentration}
If $\calD_x$ satisfies $U$-anti-concentration, then for any unit norm $\bar v$, $\phi_{\bar v}(\gamma) \leq 2U \gamma$. 
\end{example}
\begin{proof}
We can write $\P(|\bar v^\top x| \leq \gamma) = \int_{-\gamma}^\gamma p_{\bar v}(z)\mathrm dz \leq 2 \gamma U.$
\end{proof}
We will show below that log-concave isotropic distributions satisfy $U$-anti-concentration for $U = 1$.  We first remind the reader of the definition of log-concave isotropic distributions.
\begin{definition}\label{def:logconcave}
We say that a distribution $\calD_x$ over $x\in \R^d$ is \emph{log-concave} if it has a density function $p(\cdot)$ such that $\log p(\cdot)$ is concave.  We call $\calD_x$ \emph{isotropic} if its mean is the zero vector and its covariance matrix is the identity matrix.
\end{definition}
Typical examples of log-concave isotropic distributions include the standard Gaussian and the uniform distribution over a convex set.  
\begin{example}[Log-concave isotropic distributions]\label{example:logconcave.isotropic.anticoncentration}
If $\calD_x$ is log-concave isotropic then it satisfies $1$-anti-concentration, and thus for any $\bar v$ with $\norm{\bar v}=1$, $\phi_{\bar v}(\gamma) \leq 2 \gamma$.
\end{example}
\begin{proof}
This was demonstrated in~\citet[Proof of Theorem 11]{balcan2017logconcave}.\footnote{The cited theorem implies a similar bound of the form $O(\gamma)$ holds for the more general set of $s$-concave isotropic distributions. We focus here on log-concave isotropic distributions for simplicity.}
\end{proof}

\section{Gradient Descent Finds Minimizers of the Surrogate Risk}\label{sec:gd.surrogate}
We begin by demonstrating that gradient descent finds weights that achieve the best population-level surrogate risk.  The following theorem is a standard result from stochastic optimization.  For completeness, we present its proof in Appendix \ref{appendix:empirical.risk}.

\begin{theorem}
Suppose $\norm{x}\leq B_X$ a.s.  Let $\ell$ be convex, $L$-Lipschitz, and $\llp$-smooth, with $\ell(0)\leq 1$.  Let $v\in \R^d$ be arbitrary with $\norm v \leq V$ for some $V>1$, and suppose that the initialization $w_0$ satisfies $\norm{w_0}\leq V$.  For any $\eps,\delta > 0$ and for any  provided $\eta \leq (2/5)\llp^{-1} B_X^{-2}$, if gradient descent is run for $T = (4/3)\eta^{-1}\eps^{-1} \norm{w_0-v}^2$, then with probability at least $1-\delta$,
\[ \obj(w_{T-1}) \leq \obj(v) + \f{ 4B_XV L}{\sqrt n} + 8 B_XV\sqrt{\f{2\log(2/\delta)}n}.\]
\label{thm:linear.classif.obj}
\end{theorem}

This shows that gradient descent learns halfspaces that have a population surrogate risk competitive with that of the best predictor with bounded norm for any norm threshold $V$.  For distributions that are linearly separable by some margin $\gamma>0$, the above theorem allows us to derive upper bounds on the sample complexity that suggest that exponentially tailed losses are preferable to polynomially tailed losses from both time and sample complexity perspectives, touching on a recent problem posed by~\citet{ji2020regularization}.

\begin{corollary}[Sample complexity for linearly separable data]\label{corollary:linear.obj.linearly.separable}
Assume $\norm{x}\leq B_X$ a.s.  Suppose that for some $\bar v\in \R^d$, $\norm{\bar v}=1$, there is $\gamma >0$ such that $y \bar v^\top x \geq \gamma$ a.s.  If $\ell$ is convex, decreasing, $L$-Lipschitz, and $\llp$-smooth, and if we fix a step size of $\eta \leq (2/5) H^{-1} B_X^{-2}$, then
\begin{itemize}
\item Assume $\ell$ has polynomial tails, so that for some $C_0,p>0$ and $\ell(z) \leq C_0 z^{-p}$ holds for all $z\geq 1$.  Provided $n = \Omega(\gamma^{-2} \eps^{-2-2/p})$, then running gradient descent for $T = \Omega(\gamma^{-2} \eps^{-1-2/p})$ iterations guarantees that $\zeroone(w_T) \leq \eps$.
\item Assume $\ell$ has exponential tails, so that for some $C_0, C_1, p>0$, $\ell(z) \leq C_0 \exp(-C_1z^p)$ holds for all $z\geq 1$.   Then $n = \tilde \Omega(\gamma^{-2} \eps^{-2})$ and $T = \tilde \Omega(\gamma^{-2} \eps^{-1})$ guarantees that $\zeroone(w_T) \leq \eps$.
\end{itemize}
\end{corollary}
The proof for the above Corollary can be found in Appendix \ref{appendix:sample.complexity.linearly.separable}.  At a high level, the above result shows that if the tails of the loss function are heavier, one may need to run gradient descent for longer to drive the population surrogate risk, and hence the zero-one risk, to zero.\footnote{We note that in Corollary \ref{corollary:linear.obj.linearly.separable}, there is a gap for the sample complexity and runtime when using polynomially tailed vs. exponentially tailed losses.  However, such a gap may be an artifact of our analysis.  Deriving matching lower bounds for the sample complexity or runtime of gradient descent on polynomially tailed losses remains an open problem.}  In the subsequent sections, we shall see that this phenomenon persists beyond the linearly separable case to the more general agnostic learning setting.

\begin{remark}
The sample complexity in Theorem \ref{thm:linear.classif.obj} can be improved from $O(\eps^{-2})$ to $O(\eps^{-1})$ if we use online stochastic gradient descent rather than vanilla gradient descent.  The proof of this is somewhat more involved as it requires a technical workaround to the unboundedness of the loss function, and may be of independent interest.  We present the full analysis of this in Appendix \ref{appendix:sgd.fast.rate}.
\end{remark}

\section{Gradient Descent Finds Approximate Minimizers for the Zero-One Loss}\label{sec:gd.zeroone}
We now show how we can use the soft margin function to develop bounds for the zero-one loss of the output of gradient descent.  

\subsection{Bounded Distributions}\label{sec:bounded}
We first focus on the case when the marginal distribution $\calD_x$ is bounded almost surely. 

By Theorem \ref{thm:linear.classif.obj}, since by Markov's inequality we have that $\zeroone(w) \leq \ell(0)^{-1} \obj(w)$, if we want to show that the zero-one population risk for the output of gradient descent is competitive with that of the optimal zero-one loss achieved by some halfspace $v\in \R^d$, it suffices to bound $\obj(v)$ by some function of $\optz$.  To do so we decompose the expectation for $\obj(v)$ into a sum of three terms which incorporate $\optz$, the soft margin function, and a term that drives the surrogate risk to zero by driving up the margin on those samples that are correctly classified.

\begin{lemma}\label{lemma:surrogate.ub.by.zeroone}
Let $\bar v$ be a unit norm population risk minimizer for the zero-one loss, and suppose $\bar v$ satisfies the soft margin condition with respect to some $\phi:[0,1]\to\R$.  Assume that $\norm x \leq B_X$ a.s.  Let $v = V \bar v$ for $V>0$ be a scaled version of $\bar v$.  If $\ell$ is decreasing, $L$-Lipschitz and $\ell(0)\leq 1$, then 
\begin{align*}
\obj(v) &\leq \inf_{\gamma > 0} \Big \{ (1 + L V B_X ) \optz + \phi(\gamma) + \ell( V \gamma)\Big \}.
\end{align*}
\end{lemma}
\begin{proof}
We begin by writing the expectation as a sum of three terms,
\begin{align}
\nonumber
    \E[\ell(y v^\top x)] &=   \E\l[\ell(yv^\top x) \ind\l(y \bar v^\top x\leq  0\r)\r] \\ \nonumber
    &\quad + \E\l[\ell(y v^\top x) \ind\l(0 < y \bar v^\top x \leq \gamma\r)\r] \\
    &\quad + \E\l[\ell(y  v^\top x) \ind\l (y \bar v^\top x  > \gamma\r)\r] .\label{eq:objective.decomposition.3terms}
\end{align}
For the first term, we use that $\ell$ is $L$-Lipschitz and decreasing as well as Cauchy--Schwarz to get
\begin{align}
\nonumber
    \E[\ell(yv^\top x) \ind(y \bar v^\top x \leq 0)] &\leq \E[(1 + L |v^\top x | ) \ind(y \bar v^\top x \leq 0)]  \\ \nonumber
    &\leq (1 + L V B_X) \E[\ind(y \bar v^\top x \leq 0)]\\
    \nonumber
    &= (1 + L V B_X) \optz.
\end{align}
In the last inequality we use that $\norm x \leq B_X$ a.s.  
For the second term,
\begin{align} 
   \E\l[\ell(y v^\top x) \ind\l(0 < y \bar v^\top x \leq \gamma\r)\r] \leq \ell(0) \E \l[ \ind\l(0 < y \bar v^\top x \leq \gamma\r)\r]  \leq \phi(\gamma),\label{eq:surrogate.ub.by.zeroone.secondterm}
\end{align}
where we have used that $\ell$ is decreasing in the first inequality and Definition \ref{def:softmargin} in the second.  Finally, for the last term, we can use that $\ell$ is decreasing to get
\begin{align}\nonumber
    &\E\l[\ell(y v^\top x) \ind\l(y \bar v^\top x > \gamma\r)\r] \\
    &= \E\l[\ell(y V \bar v^\top x) \ind\l(y V\bar v^\top x >  V\gamma \r)\r] \leq \ell( V \gamma).    \label{eq:surrogate.ub.by.zeroone.norm1.key}
\end{align}
\end{proof}
In order to concretize this bound, we want to take $V$ large enough so that the $\ell(V\gamma)$ term is driven to zero, but not so large so that the term in front of $\optz$ grows too large. Theorem \ref{thm:linear.classif.obj} is given in terms of an arbitrary $v\in \R^d$, and so in particular holds for $v = V \bar v$.  We can view the results of Theorem~\ref{thm:linear.classif.obj} as stating an equivalence between running gradient descent for longer and for driving the norm $\norm{v}=V$ to be larger.  

We formalize the above intuition into Theorem \ref{thm:bounded} below.  Before doing so, we introduce the following notation.  For general decreasing function $\ell$, for which an inverse function may or may not exist, we overload the notation $\ell^{-1}$ by denoting $\ell^{-1}(t) := \inf \{ z : \ell(z) \leq t \}$.
\begin{theorem}\label{thm:bounded}
Suppose $\norm{x}\leq B_X$ a.s.  Let $\ell$ be convex, decreasing, $L$-Lipschitz, and $H$-smooth, with $0<\ell(0)\leq 1$.  Assume that a unit norm population risk minimizer of the zero-one loss, $\bar v$, satisfies the $\phi$-soft-margin condition for some increasing $\phi:\R\to \R$.   Fix a step size $\eta \leq(2/5) H^{-1} B_X^{-2}$.  Let $\eps_1, \gamma>0$ and $\eps_2\geq 0$ be arbitrary.  Denote by $w_T$ the output of gradient descent run for $T = (4/3) \eta^{-1} \eps_1^{-1} \gamma^{-2} [\ell^{-1}(\eps_2)]^{-2}$ iterations after initialization at the origin.  Then, with probability at least $1-\delta$, 
\begin{align*} 
\zeroone(w_T) &\leq   \ell(0)^{-1} \big[ (1 + L B_X \gamma^{-1} \ell^{-1}(\eps_2)) \optz + \phi(\gamma)  + O( \gamma^{-1} \ell^{-1}(\eps_2) n^{-1/2}) + \eps_1 + \eps_2\big] ,
\end{align*}
where $O(\cdot)$ hides absolute constants that depend on $L$, $H$, and $\log(1/\delta)$.
\end{theorem}
\begin{proof}
We take $v = V \bar v$ for a given unit-norm zero-one population risk minimizer $\bar v$ in Theorem \ref{thm:linear.classif.obj} to get that for some universal constant $C>0$ depending only on $L$ and $\log(1/\delta)$, with probability at least $1-\delta$,
\begin{equation}
    \obj(w_{T}) \leq \obj(v) + \eps_1/2 + C V B_X n^{-1/2}.
    \label{eq:objw.vs.objv}
\end{equation}
By Lemma \ref{lemma:surrogate.ub.by.zeroone}, for any $\gamma >0$ it holds that 
\begin{align}\nonumber
    \obj(v) \leq (1 + L V B_X) \optz + \phi(\gamma) + \ell(V \gamma).
\end{align}
Let now $V = \gamma^{-1} \ell^{-1}(\eps_2)$.  Then $\ell(V \gamma) = \eps_2$, and putting this together with \eqref{eq:objw.vs.objv}, we get
\begin{align}
    \obj(w_{T}) &\leq ( 1 + L \gamma^{-1} ) \optz + \phi(\gamma) + O(\gamma^{-1}  \ell^{-1} (\eps_2) n^{-1/2}) + \eps_1 + \eps_2.
    \label{eq:objv.val.agnostic}
\end{align}
Finally, by Markov's inequality,
\begin{align}
    \P(y w_{T}^\top x < 0) &\leq \f{ \E[\ell(y w_{T-1}^\top x)]}{\ell(0)}  = \frac{ \obj(w_{T})}{\ell(0)}.
    \label{eq:agnostic.markov}
\end{align}
Putting \eqref{eq:objv.val.agnostic} together with \eqref{eq:agnostic.markov} completes the proof. 
\end{proof}
A few comments on the proof of the above theorem are in order.  Note that the only place we use smoothness of the loss function is in showing that gradient descent minimizes the population risk in~\eqref{eq:objw.vs.objv}, and it is not difficult to remove the $H$-smoothness assumption to accommodate e.g., the hinge loss.   On the other hand, that $\ell$ is $L$-Lipschitz is key to the proof of Lemma \ref{lemma:surrogate.ub.by.zeroone}.  Non-Lipschitz losses such as the exponential loss or squared hinge loss would incur additional factors of $\gamma^{-1}$ in front of $\optz$ in the final bound for Theorem \ref{thm:bounded}.\footnote{This is because the first term in \eqref{eq:objective.decomposition.3terms} would be bounded by $\optz \cdot \sup_{|z|\leq V B_X} \ell(z)$.  For Lipschitz losses this incurs a term of order $O(V)$ while (for example) the exponential loss would have a term of order $O(\exp(V))$, and our proof requires $V = \Omega(\gamma^{-1})$.}  We shall see below in the proof of Proposition \ref{prop:soft.margin.bounded} that this would yield worse guarantees for $\zeroone(w_T)$.  

Additionally, in concordance with the result from Corollary \ref{corollary:linear.obj.linearly.separable}, we see that if the tail of $\ell$ is fatter, then $\ell^{-1}(\eps_2)$ will be larger and so our guarantees would be worse.  In particular, for losses with exponential tails, $\ell^{-1}(\eps_2) = O(\log(1/\eps_2))$, and so by using such losses we incur only additional logarithmic factors in $1/\eps_2$.  For this reason, we will restrict our attention in the below results to the logistic loss---which is convex, decreasing, $1$-Lipschitz and $\nicefrac 14$-smooth---although they apply equally to more general losses with different bounds that will depend on the tail behavior of the loss.

We now demonstrate how to convert the bounds given in Theorem \ref{thm:bounded} into bounds solely involving $\optz$ by substituting the forms of the soft margin functions given in Section \ref{sec:soft.margins}. 

\begin{corollary}[Hard margin distributions]\label{corollary:agnostic.margin.bounded}
Suppose that $\norm x \leq B_X$ a.s. and that a unit norm population risk minimizer $\bar v$ for the zero-one loss satisfies $|\bar v^\top x| \geq \bar \gamma > 0$ almost surely under $\calD_x$ for some $\bar \gamma>0$.   For simplicity assume that $\ell(z) = \log(1+\exp(-z))$ is the logistic loss.  Then for any $\eps,\delta>0$, with probability at least $1-\delta$, running gradient descent for $T = \tilde O( \eta^{-1} \eps^{-1} \bar \gamma^{-2})$ is guaranteed to find a point $w_{T}$ such that
\[ \zeroone(w_{T}) \leq \frac 1 {\log 2} \Big[ \optz + 2B_X \bar \gamma^{-1} \optz\log (\nicefrac 2 \optz )\Big] + \eps, \]
provided $n = \tilde \Omega(\bar \gamma^{-2} B_X^2 \log(1/\delta)\eps^{-2})$.
\end{corollary}
\begin{proof}
Since $|\bar v^\top x| \geq \gamma^* > 0$, $\phi(\gamma^*)=0$.  Note that the logistic loss is $\nicefrac 14$-smooth and satisfies $\ell^{-1}(\eps) \in [\log(1/(2\eps)),\log(2/ \eps)]$.  By taking $\eps_2 = \optz$ the result follows by applying Theorem \ref{thm:bounded} with runtime $T~=~4 \eta^{-1}~\eps^{-1}~\bar \gamma^{-2}~\log^{2}(\nicefrac 1 {2\optz})$.
\end{proof}
\begin{remark}
The bound $\tilde O(\bar \gamma^{-1} \optz)$ in Corollary \ref{corollary:agnostic.margin.bounded} is tight up to logarithmic factors\footnote{In fact, one can get rid of the logarithmic factors here and elsewhere in the paper by using the hinge loss rather than the logistic loss.  In this case one needs to modify Lemma \ref{lemma:linear.classif.empirical.risk.bound} to accomodate non-smooth losses, which can be done with runtime $O(\eps^{-2})$ rather than $O(\eps^{-1})$ by e.g.~\citet[Lemma 14.1]{shalevshwartz}.  Then we use the fact that $\ell^{-1}(0)=1$ for the hinge loss.}  if one wishes to use gradient descent on a convex surrogate of the form $\ell(y w^\top x)$. \citet[Theorem 3.1]{diakonikolas2019massart} showed that for any convex and decreasing $\ell$, there exists a distribution over the unit ball with margin $\bar \gamma>0$ such that a population risk minimizer $w^* := \mathrm{argmin}_w \E[\ell(y w^\top x)]$ has zero-one population risk at least $\Omega(\bar \gamma ^{-1} \kappa)$, where $\kappa$ is the upper bound for the Massart noise probability.  The Massart noise case is more restrictive than the agnostic setting and satisfies $\optz \leq \kappa$.  A similar matching lower bound was shown by~\citet[Proposition 1]{bendavid2012surrogate}.
\end{remark}

In the below Proposition we demonstrate the utility of having \emph{soft} margins.  As we saw in the examples in Section \ref{sec:soft.margins}, there any many distributions that satisfy $\phi(\gamma) = O(\gamma)$.  We show below the types of bounds one can expect when $\phi(\gamma) = O(\gamma^p)$ for some $p>0$.  

\begin{proposition}[Soft margin distributions]\label{prop:soft.margin.bounded}
Suppose $\norm x \leq B_X$ a.s. and that the soft margin function for a population risk minimizer of the zero-one loss satisfies $\phi(\gamma) \leq C_0 \gamma^p$ for some $p>0$.  For simplicity assume that $\ell$ is the logistic loss, and let $\eta \leq (2/5) B_X^{-2}$.    Assuming $\optz > 0$, then for any $\eps,\delta>0$, with probability at least $1-\delta$, gradient descent run for $T = \tilde O(\eta^{-1} \eps^{-1} \optz^{-2/(1+p)})$ iterations with $n = \tilde \Omega(\optz^{-2/(1+p)} \log(1/\delta) \eps^{-2})$ samples satisfies
\[ \zeroone(w_{T}) \leq \tilde O\l((C_0 +  B_X) \optz^{\f p {1+p}}\r) + \eps,\]

\end{proposition}
\begin{proof}
By Theorem \ref{thm:bounded}, we have
\begin{align}\nonumber
    \zeroone(w_{T}) \leq  \nicefrac 1 {\log 2} \Big[  \l(1 +  L B_X \gamma^{-1} \ell^{-1}(\eps_2)\r) \optz + C_0 \gamma^p  + O( \gamma^{-1} B_X \ell^{-1}(\eps_2) n^{-1/2}) + \eps_1 + \eps_2\Big].\nonumber
\end{align}
For the logistic loss, $L=1$ and $\ell^{-1}(\eps)\in [\log(1/2\eps), \log(2/\eps)]$ and so we take $\eps_2 = \optz$.  Choosing $\gamma^p = \gamma^{-1} \optz$, we get $\gamma = \optz^{1/(1+p)}$ and hence 
\begin{align}\nonumber
    \zeroone(w_{T}) \leq 2 \big(2 + B_X \optz^{- \f 1 {1+p}} \log(\nicefrac 2 \optz)\big) \optz + 2C_0 \optz^{\f 1 {1+p}} + 2\eps_1,
\end{align}
provided $n =  \Omega(\optz^{\f{-2}{1+p}}\eps_1^{-2} \log(1/\delta) \log^2(\nicefrac 1 \optz))$ and $T = 4 \eta^{-1} \eps_1^{-1} \optz^{-2/(1+p)} \log^{2} (\nicefrac 1 {2\optz})$. 
\end{proof}
By applying Proposition \ref{prop:soft.margin.bounded} to Examples \ref{example:anti.concentration} and \ref{example:logconcave.isotropic.anticoncentration} we get the following approximate agnostic learning guarantees for the output of gradient descent for log-concave isotropic distributions and other distributions satisfying $U$-anti-concentration. 

\begin{corollary}\label{corollary:bounded}
Suppose that $\calD_x$ satisfies $U$-anti-concentration and $\norm x \leq B_X$ a.s.  Then for any $\eps, \delta>0$, with probability at least $1-\delta$, gradient descent on the logistic loss with step size $\eta \leq (2/5) B_X^{-2}$ and run for $T = \tilde O( \eta^{-1} \eps^{-1} \optz^{-1})$ iterations and $n = \tilde \Omega(\optz^{-1} \log(1/\delta) \eps^{-2})$ samples returns weights $w_T$ satisfying $\zeroone(w_T) \leq \tilde O( \optz^{1/2}) +\eps$, where $\tilde O(\cdot), \tilde \Omega(\cdot)$ hide universal constant depending on $B_X$, $U$, $\log(1/\delta)$ and $\log(\nicefrac 1 \optz)$ only. 
\end{corollary}

To conclude this section, we compare our result to the variant of the \textsf{Average} algorithm, which estimates the vector $w_{\mathsf{Avg}} = d^{-1}\E_{(x,y)}[xy]$.  ~\citet{kalai08agnostichalfspace} showed that when $\calD_x$ is the uniform distribution over the unit sphere, $w_{\mathsf{Avg}}$ achieves risk $O(\optz \sqrt{\log(\nicefrac 1\optz)})$.  Estimation of $w_{\mathsf{Avg}}$ can be viewed as the output convex optimization procedure, since it is the minimum of the convex objective function $F_{\mathsf{Avg}}(w) = \E[(\sip{w}{x} - y)^2]$.  

Although $\ell(w) = (\sip{w}{x}-y)^2$ is convex, it is not decreasing and thus is not covered by our analysis.  On the other hand, this loss function is not typically used in practice for classification problems, and the aim of this work is to characterize the guarantees for the most typical loss functions used in practice, like the logistic loss.  Finally, we wish to note that the approach of soft margins is not likely to yield good bounds for the classification error when $\calD_x$ is the uniform distribution on the unit sphere.  This is because the soft margin function behavior on this distribution has a necessary dimension dependence; we provide detailed calculations for this in Appendix \ref{appendix:soft.margin.uniform}.

\subsection{Unbounded Distributions}\label{sec:unbounded.main}
We show in this section that we can achieve essentially the same results from Section \ref{sec:bounded} if we relax the assumption that $\calD_x$ is bounded almost surely to being sub-exponential. 
\begin{definition}[Sub-exponential distributions]\label{def:concentration}
We say $\calD_x$ is $C_m$\emph{-sub-exponential} if every $x\sim \calD_x$ is a sub-exponential random vector with sub-exponential norm at most $C_m$.  In particular, for any $\bar v$ with $\norm {\bar v}=1$, $\P_{\calD_x}(|\bar v^\top x| \geq t) \leq \exp(- t/C_m)$.  
\end{definition}

We show in the following example that any log-concave isotropic distribution is $C_m$-sub-exponential for an absolute constant $C_m$ independent of the dimension $d$.
\begin{example}\label{example:logconcave.concentration}
If $\calD_x$ is log-concave isotropic, then $\calD_x$ is $O(1)$-sub-exponential. 
\end{example}
\begin{proof}
By Section 5.2.4 and Definition 5.22 of~\citet{vershynin}, it suffices to show that for any unit norm $\bar v$, we have $(\E | \bar v^\top x|^p)^{1/p} \leq O(p)$.   By~\citet[Theorem 3]{balcan2017logconcave}, if we define a coordinate system in which $\bar v$ is an axis, then $\bar v^\top x$ is equal to the first marginal of $\calD_x$ and is a one dimensional log-concave isotropic distribution.  By~\citet[Theorem 5.22]{lovasz}, this implies
\[ (\E[|\bar v^\top x|^p])^{1/p}\leq 2 p \E|\bar v^\top x|\leq 2p \sqrt{\E|\bar v^\top x|^2} \leq 2p.\]
In the second inequality we use Jensen's inequality and in the last inequality we have used that $\bar v^\top x = x_1$ is isotropic.
\end{proof}

As was the case for bounded distributions, the key to the proof for unbounded distributions comes from bounding the surrogate risk at a minimizer for the zero-one loss by some function of the zero-one loss. 
\begin{lemma}\label{lemma:surrogate.ub.by.zeroone.unbounded}
Suppose $\calD_x$ is $C_m$-sub-exponential.  Denote by $\bar v$ as a unit norm population risk minimizer for the zero-one loss, and let $v = V \bar v$ for $V>0$ be a scaled version of $\bar v$.  If $\ell$ is decreasing, $L$-Lipschitz and $\ell(0)\leq 1$, then
\begin{align*} 
&\E_{(x,y)\sim \calD} \ell(y v^\top x)\leq \inf_{\gamma>0} \Big \{ \phi(\gamma) + \ell(V \gamma) + \big(1 + C_m + L V C_m \log(\nicefrac 1 \optz) \big) \optz \Big \}.
\end{align*}
\end{lemma}

\begin{proof}
We again use the decomposition \eqref{eq:objective.decomposition.3terms}, with the only difference coming from the bound for the first term, which we show here.  
Fix $\xi >0$ to be chosen later.  We can write
\begin{align}\nonumber
\E[\ell(yv^\top x) \ind(y \bar v^\top x \leq 0)] &\leq \E[(1 + LV|\bar v^\top x|) \ind(y \bar v^\top x < 0)] \\
\nonumber
& = \optz + L V \E[|\bar v^\top x| \ind(y \bar v^\top x \leq 0,\ |\bar v^\top x| \leq \xi)] \\
\nonumber
&\quad \quad + \E[|\bar v^\top x| \ind(y \bar v^\top x \leq 0,\ |\bar v^\top x| > \xi)]\\ 
\nonumber
& \leq ( 1 + L V \xi) \optz + \int_\xi^\infty \P(|\bar v^\top x| > t) \mathrm dt \\
\nonumber
& \leq ( 1 + L V \xi) \optz + \int_\xi^\infty \exp(-t / C_m) \mathrm dt \\
& = ( 1 + LV \xi) \optz + C_m\exp(-\xi/C_m).\nonumber
\end{align}
The first inequality comes from Cauchy--Schwarz, the second from truncating, and the last from the definition of $C_m$-sub-exponential.  Taking $\xi = C_m \log(\nicefrac 1 \optz)$ results in
\begin{align*}&\E[\ell(yv^\top x) \ind(y \bar v^\top x \leq 0)] \leq \l(1 + C_m + L V C_m \log(\nicefrac 1 \optz) \r) \optz.\end{align*}
\end{proof}

To derive an analogue of Theorem \ref{thm:bounded} for unbounded distributions, we need to extend the analysis for the generalization bound for the output of gradient descent we presented in Theorem \ref{thm:linear.classif.obj} to unbounded distributions.  Rather than using (full-batch) vanilla gradient descent, we instead use online stochastic gradient descent.  The reason for this is that dealing with unbounded distributions is significantly simpler with online SGD due to the ability to work with expectations rather than high-probability bounds.  It is straightforward to extend our results to vanilla gradient descent at the expense of a more involved proof by using methods from e.g.,~\citet{zhang2019relu}.

Below we present our result for unbounded distributions.  Its proof is similar to that of Theorem~\ref{thm:bounded} and can be found in Appendix \ref{appendix:unbounded}. 

\begin{theorem}\label{thm:unbounded}
Suppose $\calD_x$ is $C_m$-sub-exponential, and let $\E[\norm{x}^2]\leq B_X^2$.   Let $\ell$ be convex, $L$-Lipschitz, and decreasing with $0<\ell(0) \leq 1$.  Let $\eps_1, \gamma>0$ and $\eps_2\geq 0$ be arbitrary, and fix a step size $\eta \leq L^{-2} B_X^{-2} \eps_1/4$.  By running online SGD for $T = 2\eta^{-1} \eps_1^{-1} \gamma^{-2} [\ell^{-1}(\eps_2)]^{-2}$ iterations after initialization at the origin, SGD finds a point such that in expectation over $(x_1, \dots, x_T)\sim \calD^T$,
 \begin{align*}
 \E [\zeroone(w_{t})] \leq \nicefrac 1 {\ell(0)} \Big[ \phi(\gamma)+ \eps_1 + \eps_2 +\l(1 + C_m + L C_m \ell^{-1}(\eps_2) \gamma^{-1} \log(\nicefrac 1 \optz) \r) \optz \Big].
 \end{align*}
 \end{theorem}

The above theorem yields the following bound for sub-exponential distributions satisfying $U$-anti-concentration.  Recall from Examples \ref{example:logconcave.isotropic.anticoncentration} and \ref{example:logconcave.concentration} that log-concave isotropic distributions are $O(1)$-sub-exponential and satisfy anti-concentration with $U =1$. 
\begin{corollary}\label{corollary:log.concave.isotropic}
Suppose $\calD_x$ is $C_m$-sub-exponential with $\E[\norm{x}^2]\leq B_X^2$ and assume $U$-anti-concentration holds.  Let $\ell$ be the logistic loss and let $\eps>0$.  Fix a step size $\eta\leq B_X^{-2} \eps/16$.  By running online SGD for $T = \tilde O(\eta^{-1} \eps^{-1} C_m U^{-1} \optz^{-1})$ iterations, there exists a point $w_t$, $t<T$, such that 
\[ \E[\zeroone(w_t)] \leq \tilde O\l( ( \nicefrac{C_m}U)^{1/2} \optz^{1/2}\r)  + \eps.\]
\end{corollary}
\begin{proof}
By Example \ref{example:anti.concentration}, $\phi(\gamma) \leq 2 \gamma U$.  Since $\ell^{-1}(\eps)\in [\log(1/2\eps), \log(2/\eps)]$, we can take $\eps_2 = \optz$ in Theorem~\ref{thm:unbounded} to get
\begin{align*}
\E[\zeroone(w_t)] \leq \nicefrac{1}{\log(2)} \Big[2\gamma U + \eps +\l(2 + C_m + L C_m \gamma^{-1} \log^2(\nicefrac {2} \optz) \r) \optz \Big].
\end{align*}
This bound is optimized when $U \gamma = C_m \gamma^{-1} \optz$, i.e.,  $\gamma = U^{-1/2} C_m^{1/2} \optz^{\f 12}$.  Substituting this value for $\gamma$ we get the desired bound with $T = 2\log(2) \eta^{-1} \eps^{-1} C_m U^{-1} \optz^{-1} \log^{2}(\nicefrac{1}{2\optz})$.
\end{proof}

\begin{remark}
\citet[Theorem 1.4]{diakonikolas2020nonconvex} recently showed that if the marginal of $\calD$ over $x$ is the standard Gaussian in $d$ dimensions, for every convex, non-decreasing loss $\ell$, the minimizer $v = \mathrm{argmin}_w \obj(w)$ satisfies $\zeroone(v) = \Omega(\optz\sqrt{\log(\nicefrac 1 \optz)})$.  Thus, there is a large gap between our upper bound of $\tilde O(\opt^{1/2})$ and their corresponding lower bound.  We think it is an interesting question if either the lower bound or the upper bound could be sharpened.

We also wish to note that~\citet{diakonikolas2020nonconvex} showed that by using gradient descent on a certain bounded and decreasing non-convex surrogate for the zero-one loss, it is possible to show that  gradient descent finds a point with $\zeroone(w_T)\leq O(\optz) + \eps$.  In comparison with our result, this is perhaps not surprising: if one is able to show that gradient descent with a \emph{bounded} and decreasing loss function can achieve population risk bounded by $O(\E[\ell(y v^\top x)])$ for arbitrary $v\in \R^d$, then the same proof technique that yields Theorem \ref{thm:unbounded} from Lemma \ref{lemma:surrogate.ub.by.zeroone.unbounded} would demonstrate that $\zeroone(w_t)\leq O(\optz)$.  Since the only globally bounded convex function is constant, this approach would require working with a non-convex loss.
\end{remark}

\section{Conclusion and Future Work}\label{sec:conclusion}
In this work we analyzed the problem of learning halfspaces in the presence of agnostic label noise.  We showed that the simple approach of gradient descent on convex surrogates for the zero-one loss (such as the cross entropy or hinge losses) can yield approximate minimizers for the zero-one loss for both hard margin distributions and sub-exponential distributions satisfying an anti-concentration inequality enjoyed by log-concave isotropic distributions.  Our results match (up to logarithmic factors) lower bounds shown for hard margin distributions.  For future work, we are interested in exploring the utility of the soft margin for understanding other classification problems.

\section*{Acknowledgements}
We thank Peter Bartlett for helpful comments that led us to the result on fast rates for stochastic gradient descent.

\appendix

\section{Fast Rates with Stochastic Gradient Descent}\label{appendix:sgd.fast.rate}
In Theorem \ref{thm:linear.classif.obj}, we showed that $\obj(w_T) \leq \obj(v) + O(1/\sqrt n)$ given $n$ samples from $\calD$ by using vanilla (full-batch) gradient descent.  In this section we demonstrate that by instead using stochastic gradient descent, one can achieve $\obj(w_T) \leq O(\obj(v)) + O(1/n)$ by appealing to a martingle Bernstein bound.  We note that although the population risk guarantee degrades from $\obj(v)$ to $O(\obj(v))$, our bounds for the zero-one risk in vanilla gradient descent already have constant-factor errors and so the constant-factor error for $\obj(v)$ will not change the order of our final bounds.

The version of stochastic gradient descent that we study is the standard online SGD.  Suppose we sample $z_t = (x_t, y_t)\iid \calD$ for $t=1,\dots, T$, and let us denote the $\sigma$-algebra generated by the first $t$ samples as $\calG_t = \sigma(z_1, \dots, z_t)$.  Define
\[ \hat F_t(w) := \ell(y_t w^\top x_t),\quad \E[\hat F_t(w_t) | \calG_{t-1}] = F(w_t) = \E_{(x,y)\sim \calD} \ell(y w_t^\top x).\]
The online stochastic gradient descent updates take the form
\[ w_{t+1} := w_t - \eta \nabla \hat F_t(w_t).\]
We are able to show an improved rate of $O(\eps^{-1})$ when using online SGD.
\begin{theorem}[Fast rate for online SGD]\label{thm:online.sgd}
Assume that $\ell(\cdot) \geq 0$ is convex, strictly decreasing, $L$-Lipschitz and $H$-smooth.  Assume $\norm{x}\leq B_X$ a.s.  For simplicity assume that $w_0=0$.  Let $v\in \R^d$ be arbitrary with $\norm{v}\leq V$.  Let $\eta \leq (32 H B_X^2)^{-1}$.  Then for any $\eps, \delta >0$, by running online stochastic gradient descent for $T = O(\eps^{-1} V^2\log(1/\delta))$ iterations, with probability at least $1-\delta$ there exists a point $w_{t^*}$, with $t^*<T$, such that
\[ \zeroone(w_{t^*}) \leq O(\E[\ell(y v^\top x)]) + \eps,\]
where $O(\cdot)$ hides constant factors that depend on $L$, $H$ and $B_X$ only. 
\end{theorem}
In this section we will sketch the proof for the above theorem.  First, we note the following guarantee for the empirical risk.  This result is a standard result in online convex optimization (see, e.g., Theorem 14.13 in~\citet{shalevshwartz}). 
\begin{lemma}\label{lemma:sgd.empirical.error.bound}
Suppose that $\ell(\cdot)\geq 0$ is convex and $H$-smooth, and that $\norm{x}\leq B_X$ a.s.  Then for any $\alpha\in (0,1)$, for fixed step size $\eta \leq \alpha/(8 H B_X^2)$, and for any $T \geq 1$, it holds that
\[  \f 1 T \summm t 0 {T-1} \hat F_t(w_t)   \leq (1+\alpha) \f 1 T \summm t 0 {T-1} \hat F_t(v) + \f{\norm{w_0-v}^2}{\eta T}.\]
\end{lemma}
From here, one could take expectations and show that in expectation over the randomness of SGD, the population risk found by gradient descent is at most $(1+\alpha)\obj(v) + O(1/T)$, but we are interested in developing a generalization bound that has the same fast rate but holds with high probability, which requires significantly more work.  Much of the literature for fast rates in stochastic optimization require additional structure to achieve such results:~\citet{bartlett2006convexity} showed that the empirical risk minimizer converges at a fast rate to its expectation under a low-noise assumption;~\citet{sridharan2009fast} achieved fast rates for the output of stochastic optimization by using explicit regularization by a strongly convex regularizer;~\citet{srebro2010smoothness} shows that projected online SGD achieves fast rates when $\min_v\E[\ell(y v^\top x)]=0$.   By contrast, we show below that the standard online SGD algorithm achieves a constant-factor approximation to the best population risk at a fast rate.   We do so by appealing to the following martingale Bernstein inequality.

\begin{lemma}[\citet{beygelzimer}, Theorem 1]\label{lemma:beygelzimer}
Let $\{Y_t\}$ be a martingale adapted to the filtration $\calG_t$, and let $Y_0=0$.  Let $\{D_t\}$ be the corresponding martingale difference sequence.  Fix $T>0$, and define the sequence of conditional variance
\[ U_{T-1} := \sum_{t < T} \E[D_t^2 | \calG_{t-1}],\]
and assume that $D_t \leq R$ almost surely.  Then for any $\delta \in (0,1)$, with probability greater than $1-\delta$,
\[ Y_{T-1} \leq R \log(1/\delta) + (e-2)U_{T-1}/R.\]
\end{lemma}
We would like to take $Y_t = \sum_{\tau < t}  [ F(w_t) - \hat F_t(w_t)]$, which has martingale difference sequence $D_t=F(w_t) - \hat F_t(w_t)$.  The difficulty here is showing that $D_t\leq R$ almost surely for some absolute constant $R$.  The obvious fix would be to show that the weights $w_t$ stay within a bounded region throughout gradient descent via early stopping.  In the case of full-batch gradient descent, this is indeed possible: in Lemma \ref{lemma:linear.classif.empirical.risk.bound} we showed that $\norm{w_t-v}\leq \norm{w_0-v}$ throughout gradient descent, which would imply that $\ell(y w_t^\top x)$ is uniformly bounded for all samples $x$ throughout G.D., in which case $D_t\leq F(w_t)$ would hold almost surely.  But for online stochastic gradient descent, since we must continue to take draws from the distribution in order to reduce the optimization error, there isn't a straightforward way to get a bound on $\norm{w_t}$ to hold almost surely throughout the gradient descent trajectory.

Our way around this is to realize that in the end, our end goal is to show something of the form
\[ \zeroone(w_t) \leq O(\E[\ell(y v^\top x)]) + O(1/T),\]
since then we could use a decomposition similar to Lemma \ref{lemma:surrogate.ub.by.zeroone} to bound the right hand side by terms involving $\optz$ and a soft margin function.  Since for a non-negative $H$-smooth loss $[\ell'(z)]^2 \leq 4H \ell(z)$ holds, it actually suffices to show that the losses $\{[\ell'(y_t w^\top x_t)]^2\}_1^T$ concentrate around their expectation at a fast rate.  Roughly, this is because one would have
\begin{align} \nonumber
\min_{t<T} \E_{\calD}\l([\ell'(y w_{t}^\top x)]^2\r) &\leq \f 1 T \summm t 0 {T-1} [\ell'(y_t w_t^\top x_t)^2] + O(1/T) \\
\nonumber
&\leq \f {4H} T \summm t 0 {T-1} \ell(y_t w_t^\top x_t) + O(1/T) \\
&\leq  \f {4H} T \summm t 0 {T-1} \ell(y_t v^\top x_t) + O(1/T).
\label{eq:sgd.rough.idea}
\end{align}
To finish the proof we can then use the fact that $v$ is a fixed vector of constant norm to show that the empirical risk on the last line of \eqref{eq:sgd.rough.idea} concentrates around $O( \E[\ell(y v^\top x)])$ at rate $O(1/T)$.  For decreasing and convex loss functions, $\ell'(z)^2$ is decreasing so the above provides a bound for $\zeroone(w_t)$ by Markov's inequality.  

This shows that the key to the proof is to show that $\{\ell'(y_t w_t x_t)^2\}$ concentrates at rate $O(1/T)$.  The reason this is easier than showing concentration of $\{\ell(y_t w_t x_t)\}$ is because for Lipschitz losses, $\ell'(y_t w_t^\top x_t)^2$ is uniformly bounded regardless of the norm of $w_t$.  This ensures that the almost sure condition needed for the martingale difference sequence in Lemma \ref{lemma:beygelzimer} holds trivially.  We note that a similar technique has been utilized before for the analysis of SGD~\citep{jitelgarsky2019.polylog,cao2020generalization,frei2019resnet}, although in these settings the authors used the concentration of $\{\ell'(z_t)\}$ rather than $\{\ell'(z_t)^2\}$ since they considered the logistic loss, for which $|\ell'(z)| \leq \ell(z)$.  Since not all smooth loss functions satisfy this inequality, we instead use concentration of $\{\ell'(z_t)^2\}$. 

Below we formalize the above proof sketch.  We first show that $\{ \ell'(y_t w_t^\top x_t)^2\}$ concentrates at rate $O(1/T)$ for any fixed sequence of gradient descent iterates $\{w_t\}$.
\begin{lemma}\label{lemma:sgd.trajectory.concentration}
Let $\ell$ be any differentiable $L$-Lipschitz function, and let $z_t=(x_t,y_t)\iid \calD$.  Denote $\calG_t = \sigma(z_1, \dots, z_{t})$ the $\sigma$-algebra generated by the first $t$ draws from $\calD$, and let $\{w_t\}$ be any sequence of random variables such that $w_t$ is $\calG_{t-1}$-measurable for each $t$.  Then for any $\delta>0$, with probability at least $1-\delta$,
\begin{equation} 
\f 1 T \summm t 0 {T-1} \E_{(x,y)\sim \calD} \l(\l[\ell'(y w_t^\top x)\r]^2\r) \leq \f 4 T \summm t 0 {T-1} \l[\ell'(y_t w_t^\top x_t)\r]^2 + \f{ 4L^2 \log(1/\delta)}{T}.
\label{eq:sgd.trajectory.concentration}
\end{equation}
\end{lemma}
\begin{proof}
For simplicity, let us denote
\[ J(w) := \E_{(x,y)\sim \calD} \l(\l[\ell'(y w^\top x)\r]^2\r), \quad \hat J_t(w) := \l[\ell'(y_t w^\top x_t)\r]^2.\]
We begin by showing the second inequality in \eqref{eq:sgd.trajectory.concentration}.  Define the random variable
\begin{equation}
Y_t := \sum_{\tau < t} (J(w_\tau) - \hat J_\tau(w_\tau))
\label{eq:sgd.yt.defn}
\end{equation}
Then $Y_t$ is a martingale with respect to the filtration $\calG_{t-1}$ with martingale difference sequence $D_t := J(w_t) - \hat J_t(w_t)$.  We need bounds on $D_t$ and on $\E[D_t^2 | \calG_{t-1}]$ in order to apply Lemma \ref{lemma:beygelzimer}. 
Since $\ell$ is $L$-Lipschitz,
\begin{equation}\nonumber
D_t \leq J(w_t) = \E_{(x,y)\sim \calD}\l([-\ell' (y v^\top x)]^2\r) \leq L^2.
\label{eq:sgd.wt.Dt.ub}
\end{equation}
Similarly, 
\begin{align}
\nonumber
\E[\hat J_t(w_t)^2 | \calG_{t-1}] &=  \E\l( \l[ \ell'(y_t w_t^\top x_t)\r]^4 | \calG_{t-1} \r) \\
\nonumber
&\leq L^2 \E\l(\l[\ell'(y_t w_t^\top x_t) \r]^2| \calG_{t-1}\r)\\
\label{eq:sgd.wt.squaredincrement:Ft}
&= L^2 J(w_t).
\end{align}
In the inequality we use that $\ell$ is $L$-Lipschitz, so that $|\ell'(\alpha)| \leq L$.  We then can use \eqref{eq:sgd.wt.squaredincrement:Ft} to bound the squared increments,
\begin{align}
\nonumber
\E[D_t^2 | \calG_{t-1}] &= J(w_t)^2 - 2 J(w_t)\E[\hat J_t(w_t) | \calG_{t-1}] + \E [\hat J_t(w_t)^2 | \calG_{t-1}]\\
\nonumber
&\leq \E[\hat J_t(w_t)^2 | \calG_{t-1}]\\
\nonumber
\label{eq:sgd.wt.squaredincrement:Dt}
&\leq L^2J(w_t).
\end{align}
This allows for us to bound
\begin{equation}
U_{T-1} = \summm t 0 {T-1} \E[D_t^2 | \calG_{t-1}] \leq L^2 \summm t 0 {T-1} J(w_t).\label{eq:sgd.Ut.bound}\nonumber
\end{equation}
Lemma \ref{lemma:beygelzimer} thus implies that with probability at least $1-\delta$, we have
\begin{align}
\summm t 0 {T-1} (J(w_t) - \hat J_t(w_t)) \leq L^2 \log(1/\delta) +  (\exp(1)-2) \summm t 0 {T-1} J(w_t).
\label{eq:beygelzimer.consequence}\nonumber
\end{align}
Using that $(1-\exp(1)+2)^{-1}\leq 4$, we divide each side by $T$ and get
\begin{equation}\nonumber
\f 1 T \summm 0 t {T-1} J(w_t) \leq \f 4 T \summm t 0 {T-1} \hat J_t(w_t) + \f{ 4L^2 \log(1/\delta)}T.
\end{equation}
This completes the proof.
\end{proof}
  Next, we show that the average of $\{\ell(y_t v^\top x_t )\}$ is at most twice its mean at rate $O(1/T)$. 
\begin{lemma}\label{lemma:sgd.optimal.halfspace.concentration}
Let $\ell$ be any $L$-Lipschitz function, and suppose that $\ell(0) \leq 1$ and $\pnorm{x}2\leq B$ a.s.  Let $v\in \R^d$ be arbitrary with $\norm{v}\leq V$.  For any $\delta>0$, with probability at least $1-\delta$,
\[ \f 1 T \summm t 0 {T-1} \hat F_t(v) \leq 2F(v) + \f{2(1+LVB_X)\log(1/\delta)}{T}.\]
\end{lemma}

\begin{proof}
Let $\calG_t = \sigma(z_1, \dots, z_t)$ be the $\sigma$-algebra generated by the first $t$ draws from $\calD$.  Then the random variable $Y_t := \sum_{\tau < t} (\hat F_\tau(v) - F(v))$ is a martingale with respect to the filtration $\calG_{t-1}$ with martingale difference sequence $D_t := \hat F_t(v) - F(v)$.  We need bounds on $D_t$ and on $\E[D_t^2 | \calG_{t-1}]$ in order to apply Lemma \ref{lemma:beygelzimer}. 
Since $\ell$ is $L$-Lipschitz and $\norm{x}\leq B_X$ a.s., that $\norm{v}\leq V$ implies that almost surely,
\begin{equation}
D_t \leq \hat F_t(v)= \ell(y_t v^\top x_t) \leq (1 + L V B_X).
\label{eq:sgd.v.Dt.ub}
\end{equation}
Similarly, 
\begin{align}
\nonumber
\E[\hat F_t(v)^2 | \calG_{t-1}] &=  \E\l[ \ell(y_t v^\top x_t)^2 | \calG_{t-1} \r] \\
\nonumber
&\leq ( 1 + L V B_X) \E[\ell(y_tv^\top x_t)] \\
\label{eq:sgd:squaredincrement:Ft}
&= (1 + L VB_X) F(v).
\end{align}
In the inequality, we have used that $(x_t,y_t)$ is independent from $\calG_{t-1}$ together with \eqref{eq:sgd.v.Dt.ub}.  We then can use \eqref{eq:sgd:squaredincrement:Ft} to bound the squared increments,
\begin{align}
\nonumber
\E[D_t^2 | \calG_{t-1}] &= F(v)^2 - 2 F(v) \E[\hat F_t(v) | \calG_{t-1}] + \E [\hat F_t(v)^2 | \calG_{t-1}]\\
\nonumber
&\leq \E[\hat F_t(v)^2 | \calG_{t-1}]\\
\nonumber
\label{eq:sgd:squaredincrement:Dt}
&\leq (1 + L VB_X) F(v).
\end{align}
This allows for us to bound
\[ U_{T-1} := \summm t 0 {T-1} \E[D_t^2 | \calG_{t-1}] \leq (1 + L VB_X) T F(v).\]
Lemma \ref{lemma:beygelzimer} thus implies that with probability at least $1-\delta$, we have
\begin{align*}
\summm t 0 {T-1} ( \hat F_t(v) - F(v)) \leq (1 + L VB_X) \log(1/\delta) +  (\exp(1)-2) T F(v).
\end{align*}
Using that $\exp(1)-2\leq 1$, we divide each side by $T$ and get
\begin{equation}\nonumber
 \f 1 T \summm t 0 {T-1} \hat F_t(v) \leq 2F(v) + \f{ 2(1 + L VB_X) \log(1/\delta)}T.
\end{equation}
\end{proof}

 Finally, we put these ingredients together for the proof of Theorem \ref{thm:online.sgd}.
\begin{proof}
Since $\ell$ is convex and $H$-smooth, we can take $\alpha = 1/4$ in Lemma \ref{lemma:sgd.empirical.error.bound} to get
\begin{equation}\label{eq:sgd.empirical.error.bound.2}
    \f 1 T \summm t 0 {T-1} \hat F_t(w_t) \leq \f 5 {4T} \summm t 0 {T-1} \hat F_t(v) + \f{ V^2}{\eta T}.
\end{equation}
We can therefore bound 
\begin{align}\nonumber
    \min_{t<T} \E\l([\ell'(y w_t^\top x)]^2\r) &\leq \f 1 T \summm t 0 {T-1} \E_{(x,y)\sim \calD} \l([\ell'(y w_t^\top x)]^2\r) \\
    \nonumber
    &\leq \f 4 T \summm t 0 {T-1} [\ell'(y_t w_t^\top x_t)]^2 + \f{4L^2 \log(2/\delta)}T \\
    \nonumber
    &\leq \f {16 H} T \summm t 0 {T-1} \hat F_t(w_t) + \f{ 4 L^2 \log(2/\delta)}T\\
    \nonumber
    &\leq \f{ 20H }T \summm t 0 {T-1} \hat F_t(v) + \f{ 5L^2 \log(2/\delta) + V^2 }{\eta T} \\
    &\leq 40H F(v) + \f{40H (1+L VB_X) \eta \log(2/\delta) + 5L^2 \eta \log(2/\delta) + V^2}{\eta T}.\label{eq:sgd.final}
\end{align}
The second inequality holds since $\ell$ is $L$-Lipschitz so that we can apply Lemma \ref{lemma:sgd.trajectory.concentration}.  The third inequality uses that $\ell$ is non-negative and $H$-smooth, so that $[\ell'(z)]^2\leq 4H \ell(z)$ (see~\citet[Lemma 2.1]{srebro2010smoothness}).  The fourth inequality uses \eqref{eq:sgd.empirical.error.bound.2}, and the final inequality uses Lemma \ref{lemma:sgd.optimal.halfspace.concentration}. 

Since $\ell$ is convex and decreasing, $\ddx z \ell'(z)^2 = 2 \ell'(z) \ell''(z)\leq 0$, so $\ell'(z)^2$ is decreasing.  By Markov's inequality, this implies
\[ \P(y w_t^\top x< 0) = \P\l([\ell'(y w_t^\top x)]^2 \geq \ell'(0)^2\r) \leq  [\ell'(0)]^{-2}\E\l([\ell'(y w_t^\top x)]^2\r).\]
Substituting this into \eqref{eq:sgd.final}, this implies that with probability at least $1-\delta$, 
\[ \zeroone(w_t) \leq O(F(v)) + O(V^2 \log(1/\delta) / T).\]
\end{proof}
We note that the above proof works for an arbitrary initialization $w_0$ such that $\norm{w_0}$ is bounded by an absolute constant with high probability, e.g. with the random initialization $w_0 \iid N(0,I_d / d)$.  The only difference is that we need to replace $V^2$ with $\norm{w_0-v}^2 \leq O(V^2)$ in \eqref{eq:sgd.empirical.error.bound.2} and the subsequent lines.

\section{Soft Margin for Uniform Distribution}\label{appendix:soft.margin.uniform}  
We show here that the soft margin function for the uniform distribution on the sphere has an unavoidable dimension dependence.  Consider $x\sim \calD$ is uniform on the sphere in $d$ dimensions.  Then $x$ has the same distribution as $z/\norm{z}$, where $z\sim N(0,I_d)$ is the $d$-dimensional Gaussian.  The soft margin function on $x$ thus satisfies, for $\norm v = 1$,
\begin{align*}
    \phi(\gamma) = \P_{x}(|v^\top x| \leq \gamma) &=  \P_z\l( |v^\top z|^2/\norm{z}^2 \leq \gamma^2 \r).
\end{align*}
By symmetry, we can rotate the coordinate system so that $v = (1, 0, \dots)$, which results in $\phi(\gamma)$ taking the form
\begin{align*}
 \P \l( \f{ z_1^2}{\summ i d z_i^2} \leq \gamma^2\r)
 &= \P\l((1-\gamma^2) z_1^2 \leq \gamma^2 \textstyle{\summm i 2 d z_i^2}\r)\\
 &= \P\l(z_1^2 \leq \f{ \gamma^2}{1-\gamma^2} \textstyle\summm i 2 d z_i^2 \r)\\
 &\geq \P(z_1^2 \leq \gamma^2 \textstyle{\sum}_{i=2}^d z_i^2).
\end{align*}
Since $\gamma^2 \summm i 2 d z_i^2 = \Theta(\gamma^2 d)$ with high probability by concentration of the $\chi^2$ distribution, and since $\P(|z_1| \leq a) = \Theta(a)$ for the Gaussian, this shows that $\phi(\gamma) = \Omega(\gamma \sqrt d)$ when $\calD_x$ is uniform on the sphere.  Thus our approach of using the soft margin in Theorem \ref{thm:bounded} to derive generalization bounds will result in multiplicative terms attached to $\opt$ that will grow with $d$ for such a distribution.  

\section{Proofs for Unbounded Distributions}\label{appendix:unbounded}
In this section we prove Theorem \ref{thm:unbounded}. 
\subsection{Empirical Risk}\label{appendix:sgd.empirical}
First, we derive an analogue of Lemma \ref{lemma:linear.classif.empirical.risk.bound} that holds for any distribution satisfying $\E[\norm{x}^2]\leq B_X^2$ by appealing to online stochastic gradient descent.  Note that any distribution over $\R^d$ with sub-Gaussian coordinates satisfies $\E[\norm{x}^2]\leq B^2$ for some $B\in \R$.

We use the same notation from Section \ref{appendix:sgd.fast.rate}, where we assume samples $z_t=(x_t,y_t)\iid \calD$ for $t=1,\dots, T$, and $\calG_t := \sigma(z_1, \dots, z_t)$, and denote
\[ \hat F_t(w) := \ell(y_t w^\top x_t),\quad \E[\hat F_t(w_t) | \calG_{t-1}] = F(w_t) = \E_{(x,y)\sim \calD} \ell(y w_t^\top x).\]
The online stochastic gradient descent updates take the form
\[ w_{t+1} := w_t - \eta \nabla \hat F_t(w_t).\]
\begin{lemma}\label{lemma:sgd.subgaussian}
Suppose $\E_{\calD_x}[\norm{x}^2]\leq B_X^2$.  Suppose that $\ell$ is convex and $L$-Lipschitz.  Let $v\in \R^d$ and $\eps, \alpha\in (0,1)$ be arbitrary, and consider any initialization $w_0\in \R^d$.  Provided $\eta \leq L^{-2} B_X^{-2} \eps/2$, then for any $T\in \N$,
\[ \f 1 T \summm t 0 {T-1} \E F(w_t) \leq  F(v) + \f{ \norm{w_0-v}^2}{\eta T} + \eps.\]
\end{lemma}
\begin{proof}
The proof is very similar to that of the proof of Lemma \ref{lemma:sgd.empirical.error.bound} described in Appendix \ref{appendix:sgd.empirical}, so we describe here the main modifications.  The key difference comes from the gradient upper bound: for $g_t = \ell'(y_t w_t^\top x_t)$, instead of getting an upper bound that holds a.s. in terms of the loss, we only show that its expectation is bounded by a constant:
\begin{equation}\nonumber
\E[\norm{g_t}^2| \calG_{t-1}] \leq  \E[\ell'(y_t w_t x_t)^2 \norm{x_t}^2 | \calG_{t-1}] \leq L^2 \E[\norm{x_t}^2|\calG_{t-1}] \leq L^2 B_X^2.
    \label{eq:subgaussian.grad.ub}
\end{equation}
By convexity, $\ip{g_t}{w_t-v} \geq \hat F_t(w_t) - \hat F_t(v)$.  Thus taking $\eta = O(\eps)$, we get
\begin{align}\nonumber
  \norm{w_t-v}^2 - \E[\norm{w_{t+1}-v}^2|\calG_{t-1}] &\geq \E[ 2 \eta(\hat F_t(w_t) - \hat F_t(v)) - \eta^2 \norm{g_t}^2 | \calG_{t-1}] \\
  \nonumber
  &\geq 2 \eta( F(w_t) - F(v)) - \eta^2 L^2 B_X^2 \\
  \nonumber
  &\geq 2 \eta( F(w_t) - F(v) - \eps).
\end{align}
Taking expectations with respect to the randomness of SGD and summing from $0$ to $T-1$, we get
\[\f 1 T \summm t 0 {T-1} \E F(w_t) \leq F(v) + \f{\norm{w_0-v}^2}{\eta T} + \eps.\]
\end{proof}
We note that the above analysis is quite loose and we are aware of a number of ways to achieve faster rates by introducing various assumptions on $\ell$ and $\calD_x$; we chose the presentation above for simplicity.

With the above result in hand, we can prove Theorem \ref{thm:unbounded}.
\begin{proof}
Let $\eps_1>0$.  By taking $\eta \leq L^{-2} B_X^{-2} \eps_1 / 8$ and $T = 2 V^2 \eta^{-1} \eps_1^{-1}$, Lemma \ref{lemma:sgd.subgaussian} and Markov's inequality, this implies that there exists some $t<T$ such that
\[ \E [\zeroone(w_{t})] \leq \E[ F(w_t)] \leq \nicefrac{1}{\ell(0)} \Big[ F(v) + \f{V^2}{\eta T} + \eps_1/2 \leq F(v) + \eps_1\Big].\]
By Lemma \ref{lemma:surrogate.ub.by.zeroone.unbounded}, this implies that for any $\gamma>0$,
\[ \E [\zeroone(w_{t})] \leq \nicefrac 1 {\ell(0)} \Big[ \l(1 + C_m + L V C_m \log(\nicefrac 1 \optz) \r) \optz + \phi(\gamma) + \ell(V \gamma) + \eps_1\Big] .\]
For $\eps_2\geq 0$, by taking $V = \gamma^{-1} \ell^{-1}(\eps_2)$, this means that for any $\gamma > 0$, we have
\[ \E [\zeroone(w_{t})] \leq \nicefrac 1 {\ell(0)} \Big[ \l(1 + C_m + L C_m \ell^{-1}(\eps_2) \gamma^{-1} \log(\nicefrac 1 \optz) \r) \optz + \phi(\gamma)+ \eps_1 + \eps_2\Big] .\]
For $V= \gamma^{-1} \ell^{-1}(\eps_2)$, we need $T = 2 \gamma^{-2} \eta^{-1} \eps_1^{-1} [\ell^{-1}(\eps_2)]^{2}$.

\end{proof}

\section{Loss Functions and Sample Complexity for Separable Data}\label{appendix:sample.complexity.linearly.separable}
We present here the proof of Corollary \ref{corollary:linear.obj.linearly.separable}.
\begin{proof}
Let $v = V \bar v$.  By Theorem \ref{thm:linear.classif.obj}, for any $\eps,\delta>0$ and $V >0$, running gradient descent for $T = 4[\ell(0)]^{-1}\eta^{-1} V^2 \eps^{-1}$ iterations guarantees that $w = w_{T-1}$ satisfies
\begin{equation}\nonumber
\obj(w) \leq \obj(v) + \ell(0)\cdot \eps/3 + CV n^{-1/2},
\end{equation}
for some absolute constant $C>0$ depending only on $L$, $B_X$, and $\log(1/\delta)$.  By Markov's inequality, this implies
\begin{align}
\P(y w^\top x < 0) \leq \f 1 {\ell(0)} \obj(w) \leq \f 1 {\ell(0)} \l( \obj(v) + \f{\ell(0)}3\eps + CV n^{-1/2}\r).
\label{eq:linear.sep.zeroone}
\end{align}
Since $y \bar v^\top x \geq \gamma$ a.s., we have
\begin{align*}
    \obj(v) &= \E_{(x,y)\sim \calD} \ell(y V \bar v^\top x) \leq \ell(V \gamma).
\end{align*}
If $\ell$ has polynomial tails, then by taking $V \geq \gamma^{-1}(6C_0 [\ell(0)]^{-1} \eps^{-1})^{1/p}$ we get $\obj(v) \leq C_0 (\gamma V)^{-p} \leq \f{\ell(0)\eps}6$.  
Substituting this into \eqref{eq:linear.sep.zeroone}, this implies 
\begin{equation}
    \P(y w^\top x < 0) \leq \f{\eps}2 + \f{ C V}{\ell(0) n^{1/2}}.
    \label{eq:lin.sep.zeroone}
\end{equation}
Thus, provided $n = \Omega(\gamma^{-2} \eps^{-2 - \f{2}{p}})$, if we run gradient descent for $T = \tilde \Omega(\gamma^{-2} \eps^{-1 - \f{2}{p}})$ iterations, we have that $\zeroone(w) \leq \eps$. 

If $\ell$ has exponential tails, then by taking $V \geq \gamma^{-1} [C_1^{-1} \log(6C_0 \ell(0) \eps^{-1})]^{1/p}$ we get $\obj(v) \leq \f{\ell(0) \eps}6$, and so \eqref{eq:lin.sep.zeroone} holds in this case as well.  This shows that for exponential tails, taking $n = \tilde \Omega(\gamma^{-2} \eps^{-2})$ and $T = \tilde \Omega(\gamma^{-2} \eps^{-1})$ suffices to achieve $\zeroone(w) \leq \eps$.  
\end{proof}

\section{Remaining Proofs}\label{appendix:empirical.risk}
In this section we provide the proof of Theorem \ref{thm:linear.classif.obj}.   We first will prove the following bound on the empirical risk.

\begin{lemma}\label{lemma:linear.classif.empirical.risk.bound}
Suppose that $\ell$ is convex and $\llp$-smooth. Assume $\norm{x}\leq B_X$ a.s.  Fix a step size $\eta \leq (2/5) \llp^{-1} B_X^{-2}$, and let $v\in \R^d$ be arbitrary.  Then for any initialization $w_0$, and for any $\eps>0$, running gradient descent for $T = (4/3) \eps^{-1} \eta^{-1} \norm{w_0-v}^2$ ensures that for all $t<T$, $\norm{w_t-v} \leq \norm{w_0-v}$, and
\begin{equation}\nonumber
  \hat \obj(w_{T-1}) \leq \f 1 T \summm t0 {T-1} \hat \obj(w_t) \leq \hat \obj(v) + \eps.
\end{equation}
\end{lemma}
To prove this, we first introduce the following upper bound for the norm of the gradient.  
\begin{lemma}[\citet{shamir2020}, Proof of Lemma 3]\label{lemma:sharp.grad.ub}
Suppose that $\ell$ is $\llp$-smooth. Then for any $\rho \in (0,1)$, provided $\eta \leq 2 \rho \llp^{-1} B_X^{-2}$, $\hat \obj(w_t)$ is decreasing in $t$.  Moreover, if $T\in \N$ is arbitrary and $u\in \R^d$ is such that $\hat \obj(u) \leq \hat \obj(w_T)$, then for any $t<T$, we have the following gradient upper bound,  
\begin{equation}
    \norm{\nabla \hat\obj (w_t)}^2 \leq \f 1 {\eta(1-\rho)} \l (\hat \obj(w_t) - \hat \obj(u)\r).\label{eq:sharp.grad.ub}
\end{equation}
\end{lemma}
With this gradient upper bound, we can prove Lemma \ref{lemma:linear.classif.empirical.risk.bound}.
\begin{proof}
Let $\eps>0$ be fixed and let $T = (4/3) \eps^{-1} \eta^{-1} \norm{w_0-v}^2$ be as in the statement of the lemma.  We are done if $\hat \obj(w_T) < \hat \obj(v)$, so let us assume that $\hat \obj(v) \leq \hat \obj(w_T)$.  We proceed by providing the appropriate lower bounds for
\begin{equation}\nonumber
    \norm{w_t-v}^2 -\norm{w_{t+1}-v}^2 = 2 \eta \ip{\hat \obj(w_t)}{w_t-v} - \eta^2 \norm{\hat \obj(w_t)}^2.
\end{equation}
For any $v\in \R^d$, by convexity of $\ell$,
\begin{align}\nonumber
\ip{\nabla \hat \obj(w)}{w-v} &= \f 1n\summ in \ell'(y_i w^\top x_i) (y_iw^\top x_i - y_iv^\top x_i)\\
\nonumber
&\geq \f 1 n \summ i n[ \ell(y_i w^\top x_i) - \ell(y_i v^\top x_i) ]\\
&= \hat \obj(w) - \hat \obj(v),
\label{eq:linear.classification.ip.lb}
\end{align}
by convexity of $\ell$.  On the other hand, since $\hat \obj(v) \leq \hat \obj(w_T)$, by Lemma \ref{lemma:sharp.grad.ub}, for any $t<T$, \eqref{eq:sharp.grad.ub} holds, i.e.
\begin{equation}
    \norm{\nabla \hat\obj (w_t)}^2 \leq \f 1 {\eta(1-\rho)} \l (\hat \obj(w_t) - \hat \obj(v)\r).
    \label{eq:sharp.grad.ub.2}
\end{equation}
Thus, for $\eta \leq (2/5) \llp^{-1} B_X^{-2}$, putting eqs. \eqref{eq:linear.classification.ip.lb} and \eqref{eq:sharp.grad.ub.2} together yields
\begin{align}\nonumber
\norm{w_t-v}^2 - \norm{w_{t+1}-v}^2 &= 2 \eta\ip{\nabla \hat \obj(w_t)}{w_t-v} - \eta^2 \norm{\nabla \hat \obj(w_t)}^2\\
\nonumber
&\geq 2\eta(\hat \obj(w_t) - \hat \obj(v)) - \eta^2 \cdot \f{1}{\eta (1-1/5)} \l( \hat \obj(w_t) -\hat \obj(v)\r)\\
\nonumber
&= \f 3 4 \eta \l( \hat \obj(w_t) - \hat \obj(v)\r).\label{eq:linear.classification.decomp}
\end{align}
Summing and teloscoping over $t<T$, 
\[ \f 1 T \summm t0{T-1} \hat \obj(w_t) \leq \hat \obj(v) + \f{(4/3) \norm{w_0-v}^2}{\eta T} \leq \hat \obj(v) + \eps.\]
By Lemma \ref{lemma:sharp.grad.ub}, $\hat \obj(w_t)$ is decreasing in $t$, and therefore 
\[ \hat \obj(w_{T-1})= \min_{t<T} \hat \obj(w_t) \leq T^{-1} \sum_{t<T} \hat \obj(w_t),\]
completing the proof.
\end{proof}

Lemma \ref{lemma:linear.classif.empirical.risk.bound} shows that throughout the trajectory of gradient descent, $\norm{w_t}$ stays bounded by the norm of the reference vector $v$.  We can thus use Rademacher complexity bounds to prove Theorem \ref{thm:linear.classif.obj}.

\begin{proof}
By Lemma \ref{lemma:linear.classif.empirical.risk.bound}, it suffices to show that the gap between the empirical and population surrogate risk is small.  To do so, we use a Rademacher complexity argument. Denote by $\calG$ the function class
\[ \calG_V := \{ x \mapsto w^\top x : \norm{w} \leq 3 V\}.\]
Since $\ell$ is $L$-Lipschitz and $\ell(0)\leq 1$, it holds that $\ell(y w^\top x)\leq 1 + 3LV\leq 4LV$.  We therefore use standard results in Rademacher complexity (e.g. Theorem 26.12 of~\citeauthor{shalevshwartz},~\citeyear{shalevshwartz}) to get that with probability at least $1-\delta$, for any $w\in \calG_V$,
\begin{equation}\nonumber
\obj(w) \leq \hat \obj(w) + \f{ 2 B_X V L}{\sqrt n} + 4 B_X V \sqrt{\f{ 2 \log(2/\delta)}{n}}.
\end{equation}
Since the output of gradient descent satisfies $\norm{w_{T-1} - v} \leq \norm{w_0-v} \leq 2 V$, we see that $w_{T-1}\in \calG_V$.  We can thus apply the Rademacher complexity bound to both $w_{T-1}\in \calG_V$ and $v\in \calG_V$, proving the theorem. 
\end{proof}

\bibliographystyle{ims}

\bibliography{references}

\end{document}